\documentclass[11pt, authoryear]{article}
\usepackage{amstext,verbatim}
\usepackage{amsthm}
\usepackage{bm}
\usepackage{amssymb,amsmath,mathrsfs,dsfont,setspace}
\usepackage{graphicx}
\usepackage{array}
\usepackage{natbib}
\usepackage{epsfig}
\usepackage{amssymb}
\usepackage{sectsty}
\usepackage{undertilde}
\usepackage{algorithm}
\usepackage{graphicx,psfrag,subfigure,amsmath,amssymb,appendix,multirow,float,threeparttable,abstract}
\usepackage{mathrsfs,setspace}
\usepackage{color}
\usepackage[margin=1in]{geometry}

\fontsize{12}{0}
\sectionfont{\large}
\subsectionfont{\large}
\subsubsectionfont{\large}
\numberwithin{equation}{section}
   {\begin{center} \begin{tabular}{|@{\hspace{.15in}}c@{\hspace{.15in}}|}
                   \hline \\ \begin{minipage}[t]{\boxedparwidth}
                   \setlength{\parindent}{.25in}}%
   {\end{minipage} \\ \\ \hline \end{tabular} \end{center}}

\newtheorem{theorem}{Theorem}
\newtheorem{assumption}{Assumption}
\newtheorem{lemma}{Lemma}

\newcommand{\eonegl}{\widehat{\mathcal E}_1^\mathrm{GL}}
\newcommand{\bSigma}{{\bf \Sigma}}
\newcommand{\etwogl}{\widehat{\mathcal E}_2^\mathrm{GL}}
\newcommand{\eonesis}{\widehat{\mathcal E}_1^\mathrm{GRASS}}
\newcommand{\etwosis}{\widehat{\mathcal E}_2^\mathrm{GRASS}}
\newcommand{\ehg}{\widehat{\mathcal E}_{\gamma_n}}
\newcommand{\ehag}{\widehat{\mathcal E}_{a,\gamma_n}}
\newcommand{\maximize}{\mbox{maximize}}
\newcommand{\trace}{\mbox{trace}}

\begin{document}

\title{ {Sure  Screening for Gaussian Graphical Models}}
\author{Shikai Luo, Rui Song, Daniela Witten}
\date{\today}

\maketitle

\doublespacing

\begin{center}Abstract
\end{center}
We propose \textit{graphical sure screening}, or GRASS, a very simple and computationally-efficient screening procedure for recovering the structure of a Gaussian graphical model in the high-dimensional setting. The GRASS estimate of the conditional dependence graph is obtained by thresholding the elements of the sample covariance matrix. The proposed approach possesses the sure screening property: with very high probability, the GRASS estimated edge set contains the true edge set. Furthermore, with high probability, the size of the estimated edge set is controlled. We provide a choice of threshold for GRASS that can control the expected false positive rate.  We illustrate the performance of GRASS in a simulation study and on a gene expression data set, and show that in practice it performs quite competitively with more complex and computationally-demanding techniques for graph estimation.

\section{Introduction}

In recent years, graphical modeling has been a topic of great interest in both the scientific and statistical communities. 
Applications of graphical modeling are widespread, from computer vision to natural language processing to genomics. In particular, in genomics, graphical models have been extensively used to model gene regulatory networks,  composed of tens of thousands of genes. It is typically of interest to infer the structure of the graph based on hundreds, or at most thousands, of observations for which gene expression measurements are available.  Consequently, the setting is high-dimensional, in the sense that there are many more features than observations.

Consider the random vector $X=(X_1,\ldots,X_p)^T$, and the conditional dependence graph $\mathcal{G} = (\Gamma, \mathcal{E})$. Here $\Gamma = \{1,\dots,p\}$ is the set of nodes, and $\mathcal{E}$ is the set of edges in $\Gamma\times\Gamma$. A pair $(a,b)$ is contained in the edge set $\mathcal{E}$ if and only if $X_a$ is conditionally dependent on $X_b$, given all remaining variables $X_{\Gamma\backslash\{a,b\}} = \{ X_k; k \in \Gamma\backslash\{a,b\} \}$.

The conditional dependence graph takes  a particularly simple form if we suppose that $X\sim N({0},{\bf \Sigma})$, where  $\bf \Sigma$ is a non-singular covariance matrix.  In this setting,  a pair of variables is conditionally independent if and only if  the corresponding entry of the precision matrix $\bSigma^{-1}$ equals zero \citep{lauritzen1996graphical,mardia1980multivariate}. Consequently, in the Gaussian graphical model, recovering the edge set $\mathcal{E}$  is equivalent to recovering the sparsity pattern of the precision matrix ${\bf \Sigma}^{-1}$.

Recently, much attention has been devoted to the task of estimating and recovering the sparsity pattern of a large sparse precision matrix; we provide a brief review of the literature here.
A number of authors have considered an $\ell_1$-penalized likelihood approach in order to estimate a sparse precision matrix \citep{yuan2007model,friedman2008sparse,rothman2008sparse,ravikumar2009sparse}. To ameliorate the bias incurred by the use of an $\ell_1$ penalty, \citet{lam2009sparsistency} considered  the use of a non-convex SCAD penalty. 
Others have considered a neighborhood selection approach, which entails performing a sparse regression of each variable on all of the other variables in order to estimate the precision matrix \citep{meinshausen2006high,yuan2010high,cai2011constrained,cai2012estimating,sun2012sparse}.
 \citet{yuan2010high} considered a Dantzig-type neighborhood selection approach. 
   For many of the aforementioned approaches, statistical convergence results (in terms of various matrix norms) have been established for the high-dimensional setting. 

Although various computationally efficient algorithms for estimating a sparse precision matrix have been proposed \citep[see e.g.][]{friedman2008sparse,witten2011new}, the required computations can be burdensome  when the number of variables is in the tens of thousands,  or even higher. For example, the precision matrix for a problem with $p=25,000$ (the number of genes in the human genome) involves upwards of 300,000,000 parameters.  In such a setting, existing algorithms  tend to be infeasible.  We are thus motivated to consider a computationally-efficient screening approach for Gaussian graphical models that possesses desirable statistical properties.

In recent years, computationally simple variable screening approaches have gained popularity in the context of high-dimensional modeling. 
 \citet{fan2008sure} proposed the sure independence screening method for linear models. This approach possesses the sure screening property:  with probability going to one, all of the important variables will be selected.  \citet{fan2009ultrahigh} and \citet{fan2010sure} extended this approach to the context of generalized linear models.
Other marginal screening methods include tilting methods \citep{hall2009tilting}, generalized correlation screening \citep{hall2009using}, nonparametric screening \citep{fan2011nonparametric}, partial likelihood screening \citep{zhao2012principled}, and robust rank correlation based screening \citep{li2012feature,zhu2011model}.
Most of the existing screening methods aim to select variables by ranking utilities such as the Pearson's correlation between the marginal covariates and the response, where variables with strong marginal utilities are selected.

In this paper, we propose a novel screening procedure for recovering the structure of a Gaussian graphical model.  We call this approach \textit{graphical sure screening} (GRASS). Our approach is motivated by the fact that the $a$th column of the precision matrix ${\bf \Sigma}^{-1}$ can be obtained by regressing the $a$th feature onto the $p-1$ other features \citep{mardia1980multivariate}. This suggests that in order to estimate $\mathcal{E}_a$, the neighbourhood of the $a$th node, we can emulate the sure screening procedure of \citet{fan2008sure}  for linear models: we simply threshold the sample correlations of the $a$th feature with the $p-1$ other features. We show that  
 under certain simple assumptions, the set of nodes for which the sample correlation with the $a$th node exceeds some threshold is guaranteed to contain the true neighborhood, $\mathcal{E}_a$,  with very high probability. This property holds when the dimension $p$ grows as an  exponential function of  the  sample size $n$.
 Furthermore, we establish a surprising  connection between GRASS and existing approaches for sparse precision matrix estimation using an $\ell_1$-penalized log likelihood.

  As far as we know, this work is the first time that a sure screening procedure has been applied in an unsupervised context. The proposed method is conceptually very simple, and can be easily implemented in  very high dimensions.  In contrast to existing methods for estimating a sparse precision matrix, which typically require $\mathcal{O}(p^3)$ computations \citep{friedman2008sparse}, our procedure requires only $\mathcal{O}(p^2)$ operations, and hence can be easily scaled to large-$p$ settings.

The rest of this article is organized as follows. In Section~\ref{sec:grass}, we present the GRASS procedure. In Section~\ref{sec:theory}, we establish the theoretical properties for GRASS; these include the  sure screening property, size control of the selected edge set, and control of the theoretical false positive rate. A surprising connection between GRASS and existing $\ell_1$-penalized approaches for sparse precision matrix estimation is explored in Section~\ref{sec:l1connect}. Simulation studies are presented in Section~\ref{sec:sim}, and a real data application is  in Section~\ref{sec:real}. We close with a discussion in Section~\ref{sec:disc}.

\section{Graphical Sure Screening} \label{sec:grass}

Consider the random vector  $X =(X_1,\ldots,X_p)^T \sim N_p(0, {\bf \Sigma})$, where $\bf \Sigma$ has unit diagonals, i.e. $\mbox{E}(X_a^2)=1$ for all $a=1,\ldots,p$. The $n\times p$ data matrix ${\bf X} = \left({\bf X}_1, \cdots, {\bf X}_p\right)$ contains $n$ i.i.d. draws from $X$. Let $\gamma_n>0$ be some pre-specified threshold.

We propose to obtain a candidate edge set, $\ehg$, and a candidate neighbourhood for the $a$th node, $\ehag$, by thresholding the sample correlation matrix by $\gamma_n$. That is, we define
\begin{eqnarray}
  \ehg &=& \{(a,b): a<b, |\bm{X}_a^T\bm{X}_b|/n > \gamma_n\}, \label{g1}\\
  \ehag &=& \{b: b\neq a, |\bm{X}_a^T\bm{X}_b|/n > \gamma_n\}. \label{g2}
\end{eqnarray}
We refer to (\ref{g1}) and (\ref{g2}) as the \textit{graphical sure screening} (GRASS) estimates.

We will show in Section~\ref{sec:theory} that for an appropriate choice of $\gamma_n$,   ${\mathcal{E}}_{a}$ is contained in  $\ehag$ with very high probability, when $p$ grows exponentially with  $n$. We refer to this as the \textit{sure screening property} for the graphical model. This property holds even when the size of the selected neighbourhood is a polynomial order of the sample size, leading to a drastic decrease in the dimension. Furthermore, under certain conditions, GRASS can also control the  false positive rate.

\section{Theoretical Properties}\label{sec:theory}

Here we present some theoretical properties of the GRASS procedure. Proofs are in the Appendix. In what follows, we use the notation $\sigma_{ab} \equiv \mbox{E}(X_a X_b)$.

To begin, we introduce an assumption on the minimum correlation between two nodes connected by an edge.
\par
\begin{assumption}\label{a1} For some constants $C_1>0$ and $0 < \kappa < 1/2$, 
$$\mathop{\mathrm{min}}_{(a,b)\in \mathcal{E}}|\sigma_{ab}|\geq C_1n^{-\kappa}.$$ 
\end{assumption}

Assumption~\ref{a1} allows us to establish the sure screening property of GRASS, which is presented in Theorem~\ref{thm1}.

\begin{theorem}\label{thm1}
  Assume that Assumption~\ref{a1} holds, and that $\mathrm{log}(p) = C_3n^{\xi}$ for some constants $C_3>0$ and $\xi \in (0,1-2\kappa)$. {Let $\gamma_n = 2/3 C_1 n^{-\kappa}$}. Then, there exist constants $C_4$ and $C_5$ such that
  \begin{eqnarray*}
    P(\mathcal{E}\subseteq \ehg) & {\geq}& 1 - C_4\mathrm{exp}(-C_5n^{1-2\kappa}),~\mathrm{and}\\
    P(\mathcal{E}_a\subseteq \ehag) &{\geq}& 1 - C_4\mathrm{exp}(-C_5n^{1-2\kappa}).
    \end{eqnarray*}
\end{theorem}

\bigskip
Theorem~\ref{thm1} guarantees that with very high probability, the true edge set is contained in $\ehg$, the edge set estimate from the GRASS procedure. In other words, with very high probability, GRASS will not result in false negatives. 
 This raises the following question: how  large is $\ehag$,  the estimated neighbourhood  for the $a$th node? In order to answer this question, we must first introduce an additional assumption.

\begin{assumption} \label{a2} There exist constants $\tau\geq0$ and $C_2>0$ such that
$$\lambda_{\mathrm{max}}(\bSigma)\leq C_2n^{\tau},$$ where $\lambda_{\mathrm{max} }(\bSigma)$ is the maximal eigenvalue of  $\bSigma$.\end{assumption}
Assumption \ref{a2} indicates that the largest eigenvalue of the population covariance matrix $\bSigma$ is allowed to diverge as $n$ grows; however, it cannot diverge too quickly. This condition naturally appears in many applications. For example, it holds for the covariance matrix of a stationary time series \citep{fan2008sure}.

We now present Theorem~\ref{thm2}, which allows us to control the size of $\ehag$. 
\begin{theorem}\label{thm2}
 {Let $\gamma_n = 2/3 C_1 n^{-\kappa}$}.  Under Assumptions \ref{a1}-\ref{a2},  if $log(p)=C_3 n^{\xi}$, for $\xi \in (0,1-2\kappa)$,  then
  \begin{eqnarray*}
    P\left[|\widehat{\mathcal{E}}_{a,\gamma_n}|\leq O(n^{2\kappa+\tau})\right] \geq 1 - C_4\mathrm{exp}(-C_5n^{1-2\kappa}),
  \end{eqnarray*}
  where the constants $C_4$ and $C_5$ are as in Theorem \ref{thm1}.
\end{theorem}

{Next we propose a choice of $\gamma_n$ that allows us to control the {expected} false positive rate at a prespecified value. The false positive rate is defined as}
$$\frac{|\ehg \cap \mathcal{E}^c |}{|\mathcal{E}^c|}.$$ 
{We would like the false positive rate to decrease to 0 as $p_n$ increases with $n$. As in \cite{zhao2012principled}, we do this by first fixing the number of false positives $f$ that we are willing to tolerate; this corresponds to a false positive rate of $f / |\mathcal{E}^c|$. }
 In order to control the {expected} false positive rate, we introduce an additional assumption.
\begin{assumption}\label{a3}
For the same $\xi$ as in Theorem~\ref{thm1},
$$ \mathop{\mathrm{max}}_{(a,b)\not\in \mathcal{E}}|\sigma_{ab}| = o(n^{-\frac{1-\xi}{2}}).$$  
\end{assumption}

\begin{theorem}\label{thm4}
 Under Assumptions \ref{a1}-\ref{a3}, if $\log(p)=C_3 n^\xi$ for $\xi$ as in Theorem~\ref{thm1}, then 
   we can control the asymptotic {expected} false positive rate at $f / |\mathcal{E}^c|$  by choosing $\gamma_n = \Phi^{-1} (1- \frac{f}{p(p-1)}) / \sqrt{n}$. Furthermore, with this threshold, the sure screening property of Theorem~\ref{thm1} still holds.
\end{theorem}

\section{Connection with the Graphical Lasso}\label{sec:l1connect}

In recent years, many quite sophisticated approaches for obtaining a sparse estimate of $\bSigma^{-1}$ have been proposed. Perhaps the best-known among these is the \emph{graphical lasso} \citep{friedman2008sparse,yuan2007model}, which is the solution to the optimization problem
\begin{equation}
\maximize_{\bf \Theta}
 \{ \log \det {\bf \Theta} - \trace ( ({\bf X}^T {\bf X}/n)  {\bf \Theta}) - \lambda \sum_{i \neq j} | \Theta_{ij}| \}.
\label{eq:gl}
\end{equation} 
Recently, \citet{witten2011new} and \citet{mazumder2012exact} established a surprising result: the connected components of the graphical lasso estimator are exactly the same as the connected components that result from hard-thresholding the matrix ${\bf X}^T {\bf X}/n$ by $\lambda$. In other words, \emph{the connected components of the graphical lasso estimator are the same as the connected components of the GRASS edge set estimate, $\ehg$,} when $\gamma_n=\lambda$. 

Does this mean that GRASS and the graphical lasso are identical? Not quite. Though the connected components of the graphical lasso and of GRASS are the same, their entire sparsity patterns are, in general, \emph{not}  the same. 
 The graphical lasso can be thought as a two-stage procedure, in which we first perform GRASS with $\gamma_n=\lambda$, and then perform a smaller graphical lasso problem on each connected component of $\ehg$. 
 
Under a set of assumptions explored by \citet{ravikumar2011high}, the graphical lasso is known to be model selection consistent. Since the connected components of GRASS and the connected components of the graphical lasso are the same provided that $\gamma_n=\lambda$, this means that the model selection consistency results of \citet{ravikumar2011high} are inherited by the connected components of GRASS. In other words, the connected components of GRASS are selected consistently. 

\section{Simulation Studies}\label{sec:sim}

 \subsection{Data Generation} \label{sec:datagen}
Let $p$ denote the number of features, and $n$ the number of observations. We considered three ways of generating the edge set $\mathcal{E}$:
 \begin{list}{}{}
 \item[\emph{Simulation A: A sparse graph.}] For all $i<j$, we set $(i,j) \in \mathcal{E}$ with probability $0.01$.
       \item[\emph{Simulation B: A  graph with ten densely connected components.}] We partitioned the $p$ features into $10$ equally-sized and non-overlapping sets: $C_1 \cup C_2 \cup \ldots \cup C_{10} = \{1,\ldots,p \}$, $|C_k|=p/10$, $C_k \cap C_j = \emptyset$.  
    For all $i \in C_k, j \in C_k, i<j$, we set $(i,j) \in \mathcal{E}$. 
    \item[\emph{Simulation C: A banded graph.}] For $|i-j| \leq 2$ we set $(i,j) \in \mathcal{E}$. Otherwise, $(i,j) \notin \mathcal{E}$.
 \end{list}
 Once the edge set $\mathcal{E}$ was generated, we created a precision matrix via the following steps:
 \begin{list}{}{}
 \item{Step 1:} We generated a $p \times p$ matrix $\bf A$, where 
 $$ A_{ij} = A_{ji} = \begin{cases} 1 & \mathrm{\; for \;} i=j  \\ \mathrm{Unif}[-0.3,0.7] & \mathrm{\; for \;} (i,j) \in \mathcal{E}  \\ 0 & \mathrm{\; otherwise} \end{cases}.$$
 \item{Step 2:} We created a positive definite matrix ${\bf \Sigma}^{-1}$:
  $${\bf \Sigma}^{-1} = {\bf A} + (0.1 - \lambda_{\min}({\bf A}))  {\bf I},$$
 where $\lambda_{\min}(\bf A)$ denotes the smallest eigenvalue of $\bf A$, and $\bf I$ denotes the $p \times p$ identity matrix.
 \end{list}
 Then the covariance matrix $\bf \Sigma$ was rescaled to have diagonal elements equal to 1.
  Finally, we generated $n$ observations i.i.d. from a $N(0, {\bf \Sigma})$ distribution.

\subsection{Control of False Positive Rate}

 Theorem~\ref{thm4} states that under certain conditions, performing GRASS with 
 $\gamma _n= \frac{1}{\sqrt{n}} \Phi^{-1}(1-q/2)$ leads to control of the  expected false positive rate (FPR) at level  $q \equiv f/|\mathcal{E}^c|$.  We now investigate the extent to which GRASS controls the FPR in practice. 
Results for Simulations A-C are shown in Table~\ref{tab:fprfnr}.

Assumption~\ref{a3}, required for Theorem~\ref{thm4} to hold, states that $\max_{(a,b) \notin \mathcal{E}} |\sigma_{ab}| \rightarrow 0$ as $n \rightarrow \infty$. Simulation B satisfies this assumption,  since ${\bf \Sigma}^{-1}$ is block diagonal with completely dense blocks (and thus the same is true of $\bf \Sigma$). As expected, the FPR is controlled successfully in Simulation B (Table~\ref{tab:fprfnr}).

However, Assumption~\ref{a3} seems not to be satisfied by Simulations A and C. But Table~\ref{tab:fprfnr} reveals that the FPR is approximately controlled in these settings, especially for larger values of $q$.
How can this be?

In order to investigate this, we consider the off-diagonal elements of ${\bf \Sigma}^{-1}$ and $\bf \Sigma$ under Simulations A-C. These are displayed in Figure~\ref{fig:covprec}. We see that even for Simulations A and C, the vast majority of the large off-diagonal elements of ${\bf \Sigma}$ correspond to non-zero elements of ${\bf \Sigma}^{-1}$.

Furthermore, in Simulation A, there is  a very pronounced relationship between the values of the non-zero elements of ${\bf \Sigma}^{-1}$, and the corresponding values of $\bf \Sigma$. This is the case because in Simulation A,  ${\bf \Sigma}^{-1}$ was generated to be so sparse (Section~\ref{sec:datagen}) that, with high probability, a given column of ${\bf \Sigma}^{-1}$ contains no more than one non-zero off-diagonal element. Consequently, ${\bf \Sigma}^{-1}$ is (approximately) a block-diagonal matrix with blocks containing no more than two features. And consequently the sparsity patterns of ${\bf \Sigma}^{-1}$ and ${\bf \Sigma}$ are almost  identical. Furthermore, there is a simple monotone relationship between most of the non-zero elements of the two matrices, which can be easily derived using the standard formula for the inverse of a $2 \times 2$ matrix.

\begin{figure}[h!]
  \centering
  Simulation A \hspace{29mm} Simulation B \hspace{29mm} Simulation C \\
  \vspace{-4mm}
    \includegraphics[width=0.32\textwidth]{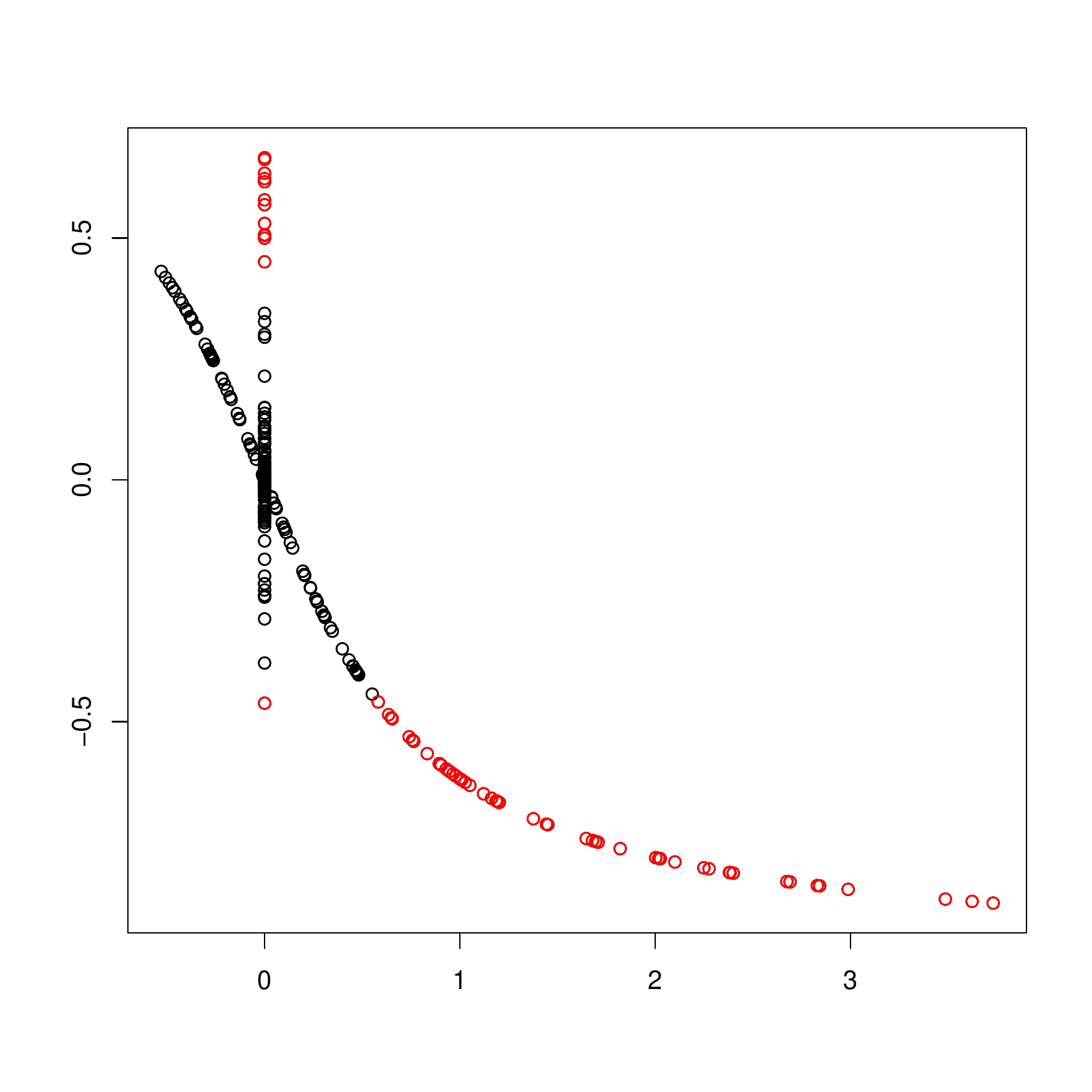}
        \includegraphics[width=0.32\textwidth]{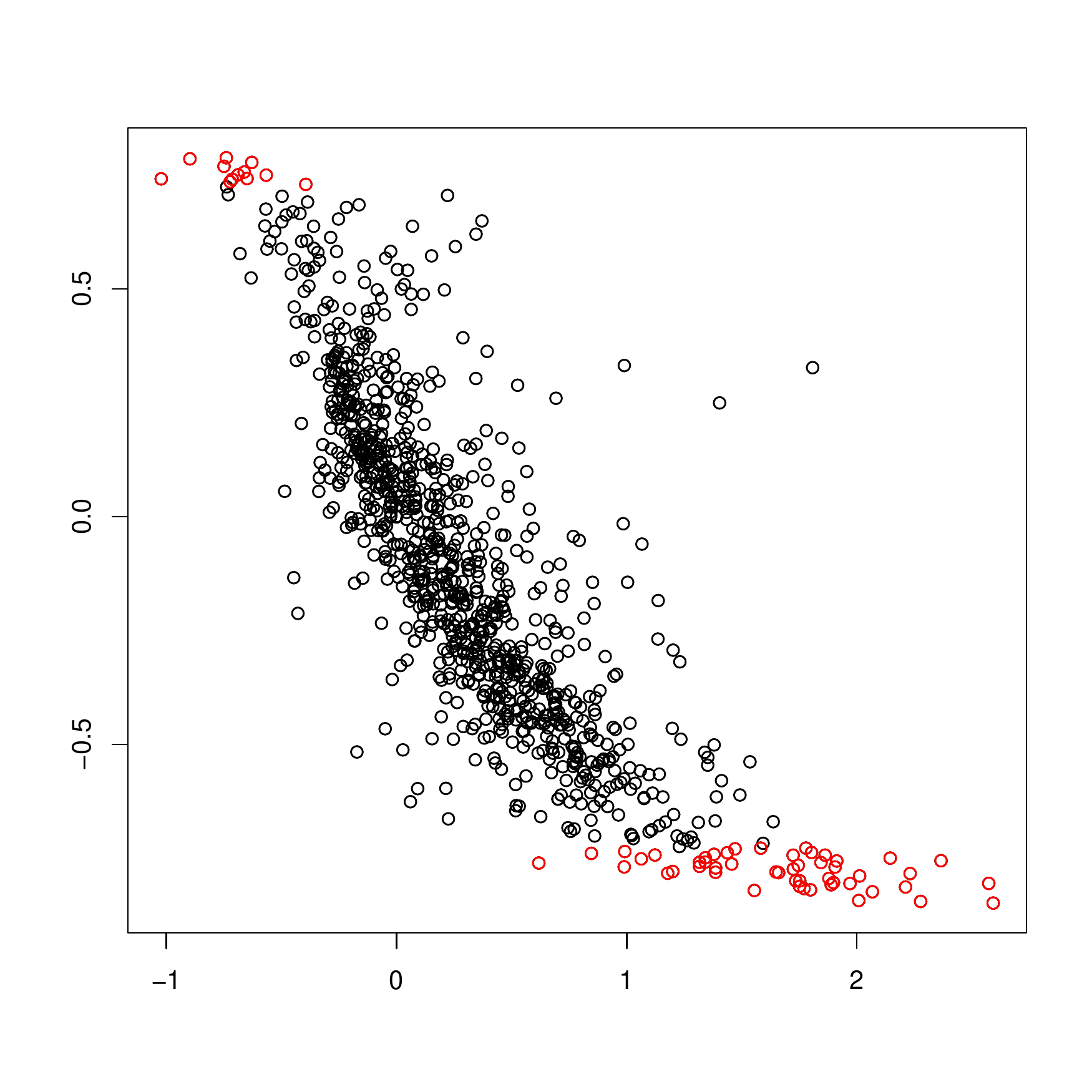}
    \includegraphics[width=0.32\textwidth]{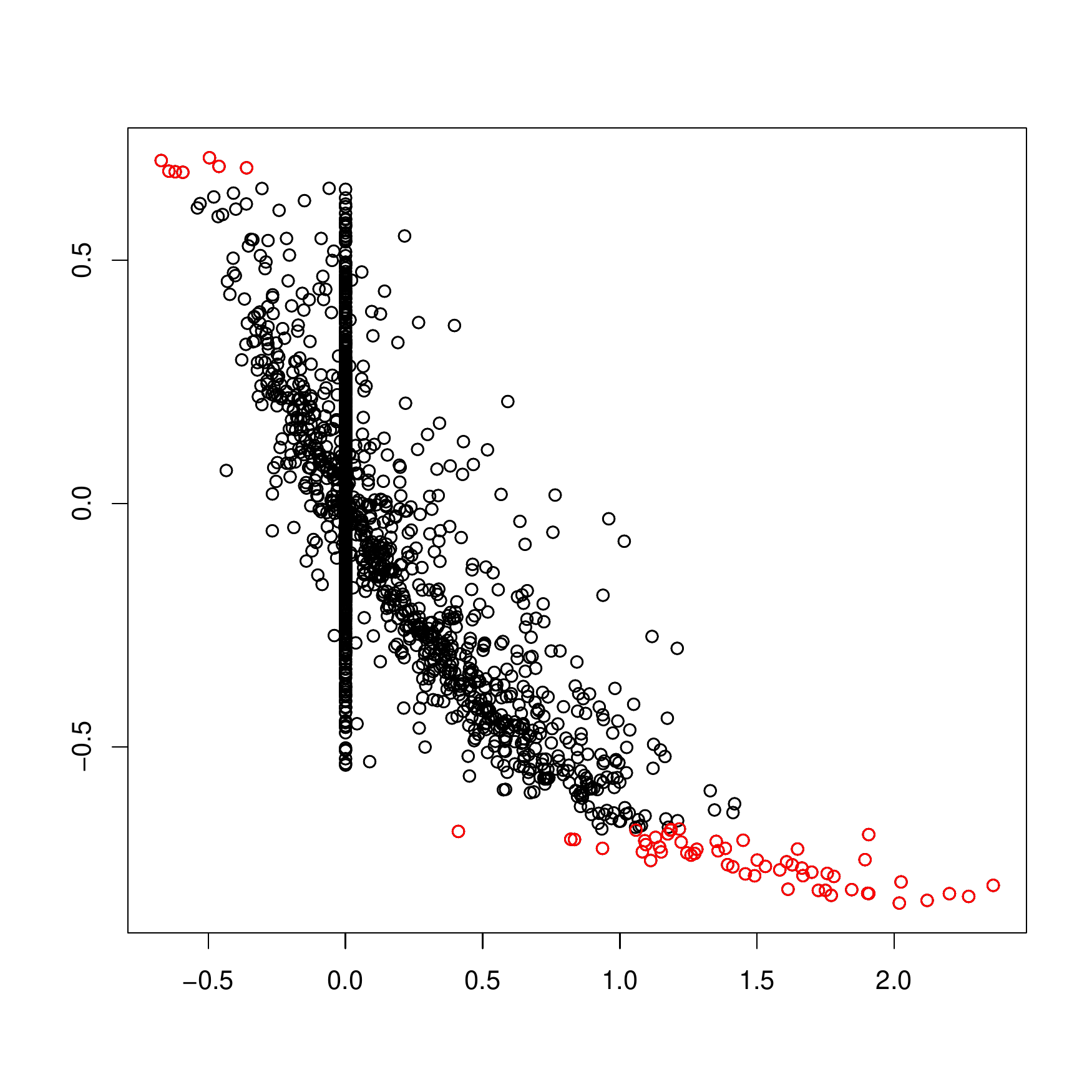}
     \caption{   \label{fig:covprec} For Simulations A-C with $n=100$ and $p=50$, the off-diagonal elements of ${\bf \Sigma}^{-1}$ (\emph{x-axis}) and $\bf \Sigma$ (\emph{y-axis}) are shown. The $0.5$\% of largest absolute off-diagonal elements of $\bf \Sigma$ are shown in red; the rest are in black. For all three setups, the vast majority of large off-diagonal elements of $\bf \Sigma$ correspond to non-zero elements of ${\bf \Sigma}^{-1}$. The pronounced relationships seen for Simulation A are due to the extreme sparsity of ${\bf \Sigma}^{-1}$, as is discussed in the text.}
\end{figure}

\begin{table}[h!]
  \begin{center}
    \begin{tabular}{cc|c|ccc|ccc|ccc|}
    \hline
&&&    \multicolumn{3}{|c|}{Simulation A}& \multicolumn{3}{|c|}{Simulation B}& \multicolumn{3}{|c|}{Simulation C}\\
 \cline{4-12} 
$n$	 &$p$	 &$q$	 &$|\ehg|$	 &FPR	 &FNR	 &$|\ehg|$	 &FPR	 &FNR	 &$|\ehg|$	 &FPR	 &FNR	 \\
\hline
\multirow{7}{*}{100} & \multirow{7}{*}{50}    &1e-04  &14.504         &0.001  &0.512  &69.816         &0      &0.652  &80.696         &0.008  &0.674  \\
  &  &0.001  &19.664         &0.003  &0.449  &88.496         &0.001  &0.572  &108.496        &0.013  &0.592  \\
  &  &0.01   &46.36  &0.013  &0.353  &131.688        &0.011  &0.464  &168.016        &0.03   &0.485  \\
  &   &0.1    &268.88         &0.103  &0.23   &368.064        &0.102  &0.305  &424.048        &0.13   &0.322  \\
  &  &0.2    &513.632        &0.204  &0.186  &603.432        &0.201  &0.239  &664.824        &0.231  &0.255  \\
  &   &0.3    &752.048        &0.302  &0.148  &834.808        &0.299  &0.194  &893.064        &0.328  &0.207  \\
  &   &0.5    &1236.6         &0.501  &0.094  &1296.344       &0.499  &0.126  &1341.68        &0.52   &0.135  \\
\hline
\multirow{7}{*}{100} & \multirow{7}{*}{200}     &1e-04  &150.664        &0.001  &0.73   &510.384        &0      &0.867  &247.328        &0.001  &0.741  \\
&   &0.001  &252.704        &0.003  &0.651  &798.528        &0.001  &0.802  &383.824        &0.003  &0.656  \\
&    &0.01   &737.752        &0.014  &0.541  &1572.592       &0.011  &0.691  &907.264        &0.014  &0.54   \\
&    &0.1    &4499.056       &0.108  &0.363  &5599.04        &0.101  &0.485  &4661.928       &0.107  &0.361  \\
&    &0.2    &8465.392       &0.208  &0.286  &9539.712       &0.2    &0.388  &8608.848       &0.206  &0.284  \\
&    &0.3    &12405.288      &0.307  &0.234  &13385  &0.3    &0.318  &12516.648      &0.305  &0.23   \\
&    &0.5    &20233.504      &0.505  &0.153  &20976.56       &0.499  &0.21   &20310.848      &0.503  &0.151  \\
\hline
\multirow{7}{*}{1000} & \multirow{7}{*}{50}   &1e-04  &27.6   &0.003  &0.187  &154.344        &0      &0.229  &244.904        &0.043  &0.243  \\
&   &0.001  &31.32  &0.004  &0.162  &163.584        &0.001  &0.193  &274.416        &0.053  &0.203  \\
&     &0.01   &55.016         &0.014  &0.129  &191.912        &0.01   &0.15   &332.968        &0.075  &0.158  \\
&     &0.1    &275.04         &0.104  &0.084  &403.752        &0.099  &0.095  &580.152        &0.18   &0.101  \\
&     &0.2    &518.92         &0.204  &0.063  &634.128        &0.2    &0.074  &803.496        &0.277  &0.08   \\
&    &0.3    &760.016        &0.304  &0.053  &862.064        &0.3    &0.06   &1016.936       &0.37   &0.063  \\
&    &0.5    &1244.624       &0.503  &0.034  &1313.84        &0.498  &0.039  &1430.144       &0.552  &0.042  \\
\hline
\multirow{7}{*}{1000} & \multirow{7}{*}{200}     &1e-04  &586.24         &0.008  &0.276  &2303.16        &0      &0.395  &855.904        &0.007  &0.279  \\
&    &0.001  &720.152        &0.011  &0.233  &2553.112       &0.001  &0.338  &992.496        &0.01   &0.237  \\
&    &0.01   &1238.992       &0.023  &0.184  &3149.528       &0.01   &0.266  &1494.464       &0.022  &0.186  \\
&  &0.1    &4992.952       &0.118  &0.118  &6753.224       &0.1    &0.171  &5181.096       &0.115  &0.118  \\
&   &0.2    &8946.776       &0.218  &0.093  &10497.4        &0.2    &0.133  &9094.128       &0.215  &0.092  \\
&    &0.3    &12846.408      &0.317  &0.076  &14198.992      &0.3    &0.108  &12970.448      &0.314  &0.074  \\
&    &0.5    &20578.888      &0.513  &0.05   &21534.512      &0.5    &0.07   &20659.576      &0.51   &0.049  \\
    \hline
    \end{tabular}
  \end{center}
 \caption{The false positive rate (FPR; defined as FP/(FP+TN)) and false negative rate (FNR; defined as FN/(TP+FN)) are reported for various values of $q$, the level of desired FPR control.  Results for two values of $n$ and $p$, and for each of Simulations A-C, are reported. The value of $|\ehg|$ --- the number of elements in the GRASS estimate corresponding to the value of $q_n$ --- is also reported.  Results are averaged over 250 simulated data sets.  \label{tab:fprfnr}}
\end{table}

\subsection{Comparison to Existing Approaches}

We now compare the performances of 
 the graphical lasso \citep{friedman2008sparse}, neighborhood selection \citep{meinshausen2006high}, and GRASS on Simulations A-C, with $n=50$ and $p=200$. 
 Results are displayed in Figure~\ref{results:abc}. 
 
 In Simulation B, recall that the sparsity patterns of $\bf \Sigma$ and ${\bf \Sigma}^{-1}$ are identical. In this setting, GRASS outperforms the graphical lasso and neighborhood selection, since it correctly assumes that the sparsity patterns of the covariance and precision matrices are similar.  
 In Simulations A and C, even though this assumption does not hold exactly, the performance of GRASS is quite competitive. 
  
 Overall, in Figure~\ref{results:abc}, there is little difference in performance between GRASS, neighbourhood selection, and the lasso. That is, even though GRASS is an extremely simple approach, in practice GRASS performs competitively with specialized and computationally-intensive procedures for estimating a precision matrix. 
 
 Figure~\ref{results:heatmaps} displays the adjacency matrices corresponding to the true edge set, as well as the edge sets estimated using the graphical lasso and GRASS, for Simulations B and C. We find once again that the edge sets estimated by the graphical lasso and GRASS appear to be quite similar. 
 
 \begin{figure}[h!]
  \centering
  Simulation A \hspace{29mm} Simulation B \hspace{29mm} Simulation C \\
  \vspace{-7mm}
    \includegraphics[width=0.32\textwidth]{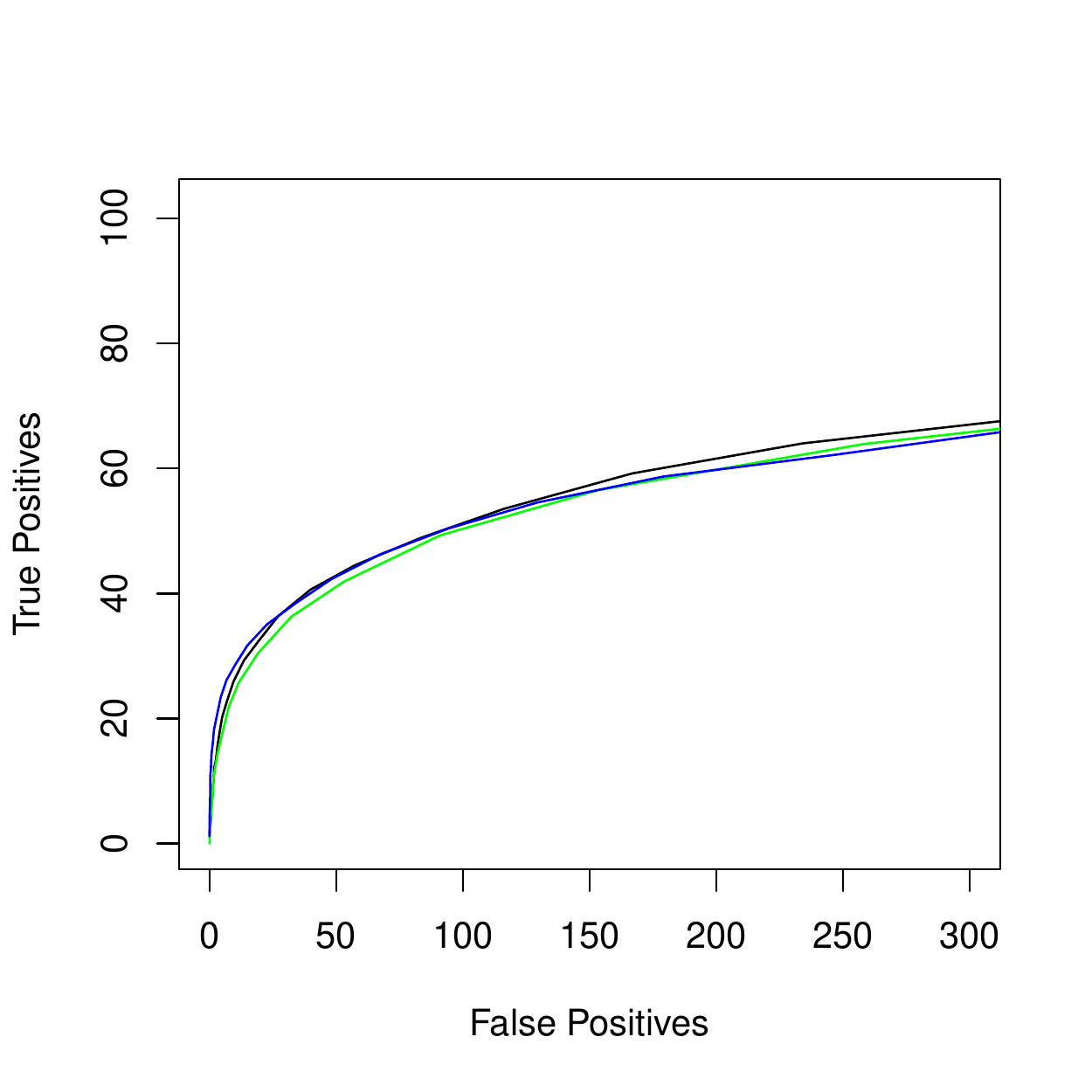}
     \includegraphics[width=0.32\textwidth]{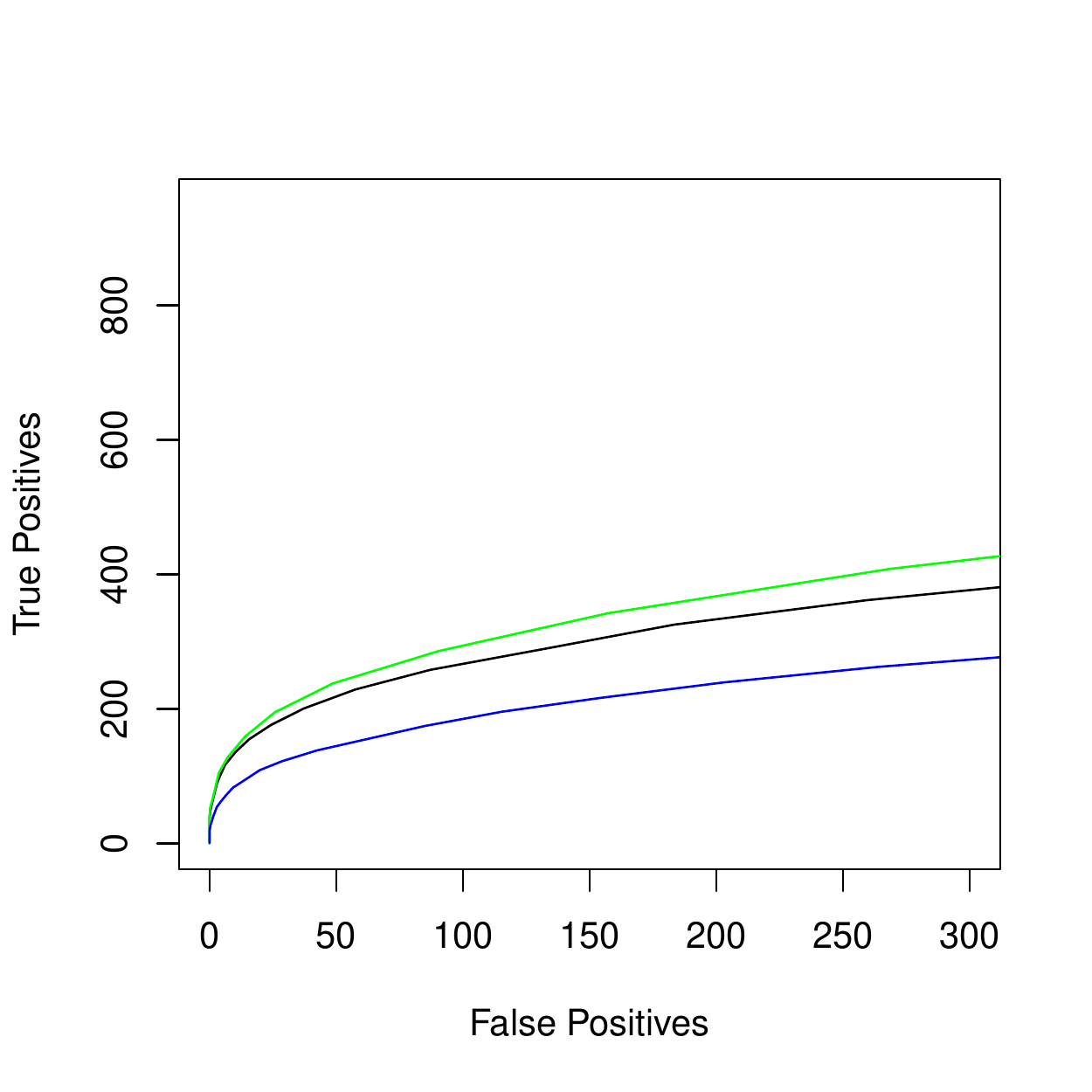}
      \includegraphics[width=0.32\textwidth]{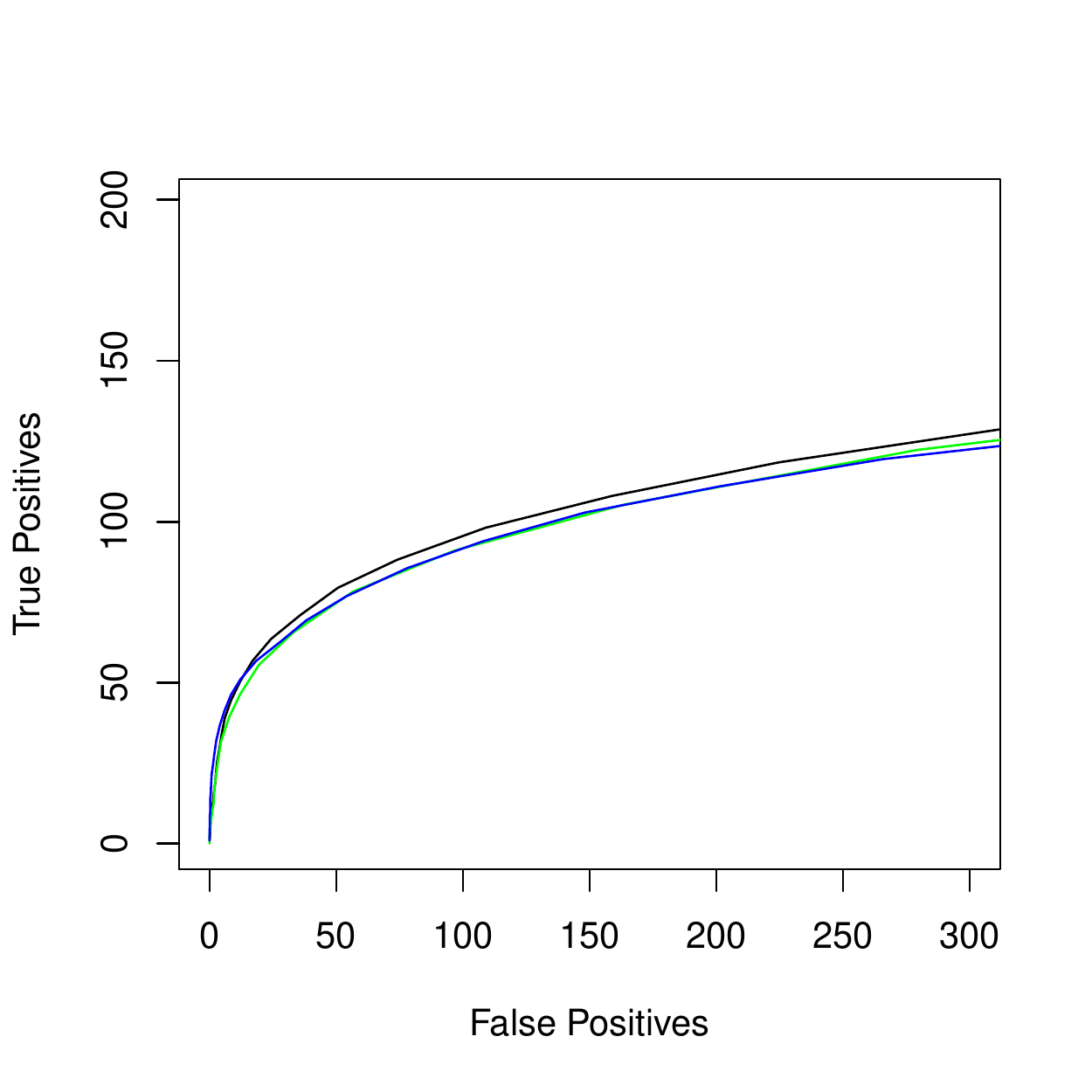}
       \caption{   \label{results:abc} For Simulations A-C with $p=200$ and $n=50$, the number of false and true positive edges detected  is displayed as the tuning parameter is varied.  Results are shown for graphical lasso (\textcolor{black}{---}), neighborhood selection (\textcolor{blue}{---}), and GRASS (\textcolor{green}{---}). }
\end{figure}

 \begin{figure}[h!]
  \centering
  (a) \hspace{44mm} (b) \hspace{44mm} (c)\\
  \vspace{2mm}
    \includegraphics[width=0.32\textwidth]{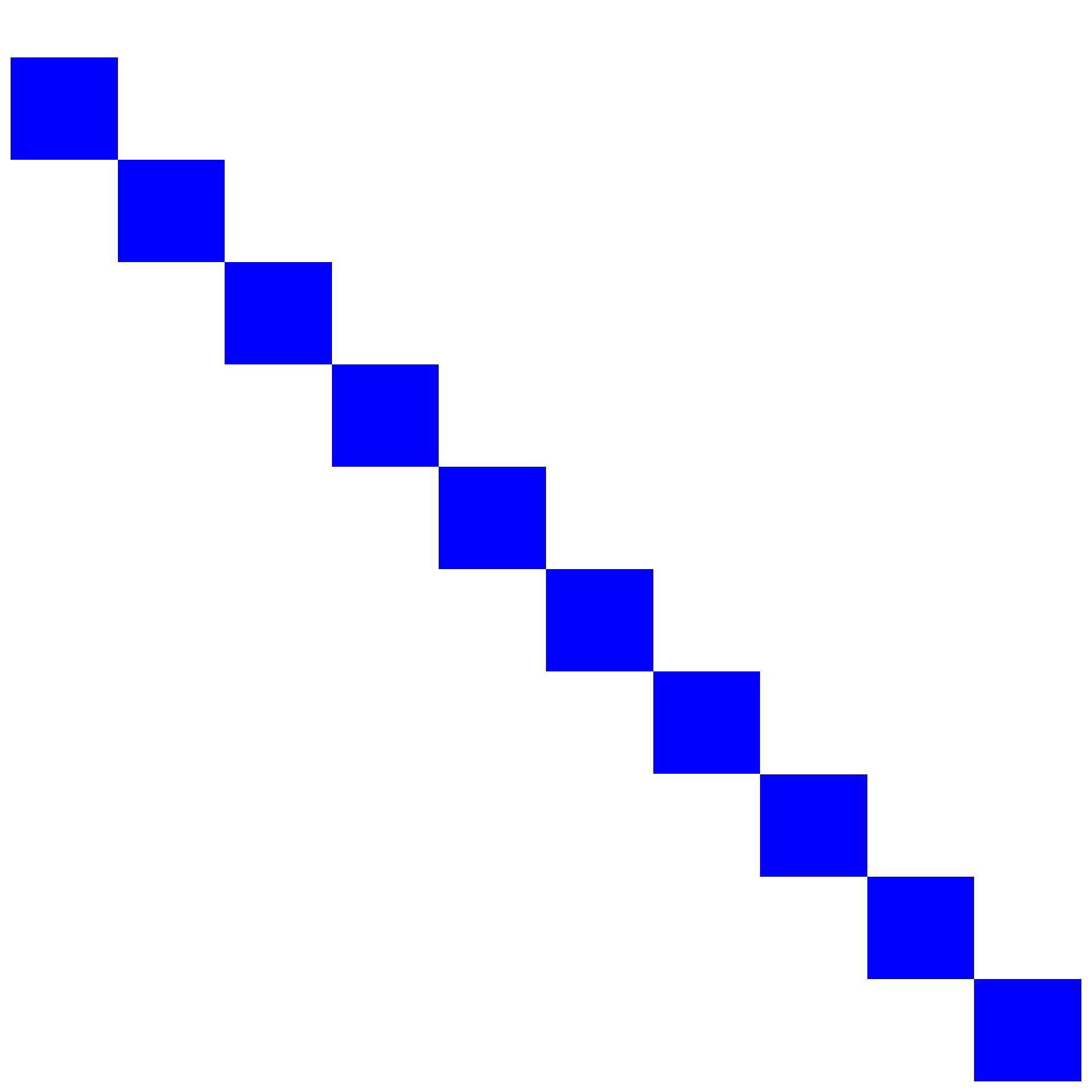}
     \includegraphics[width=0.32\textwidth]{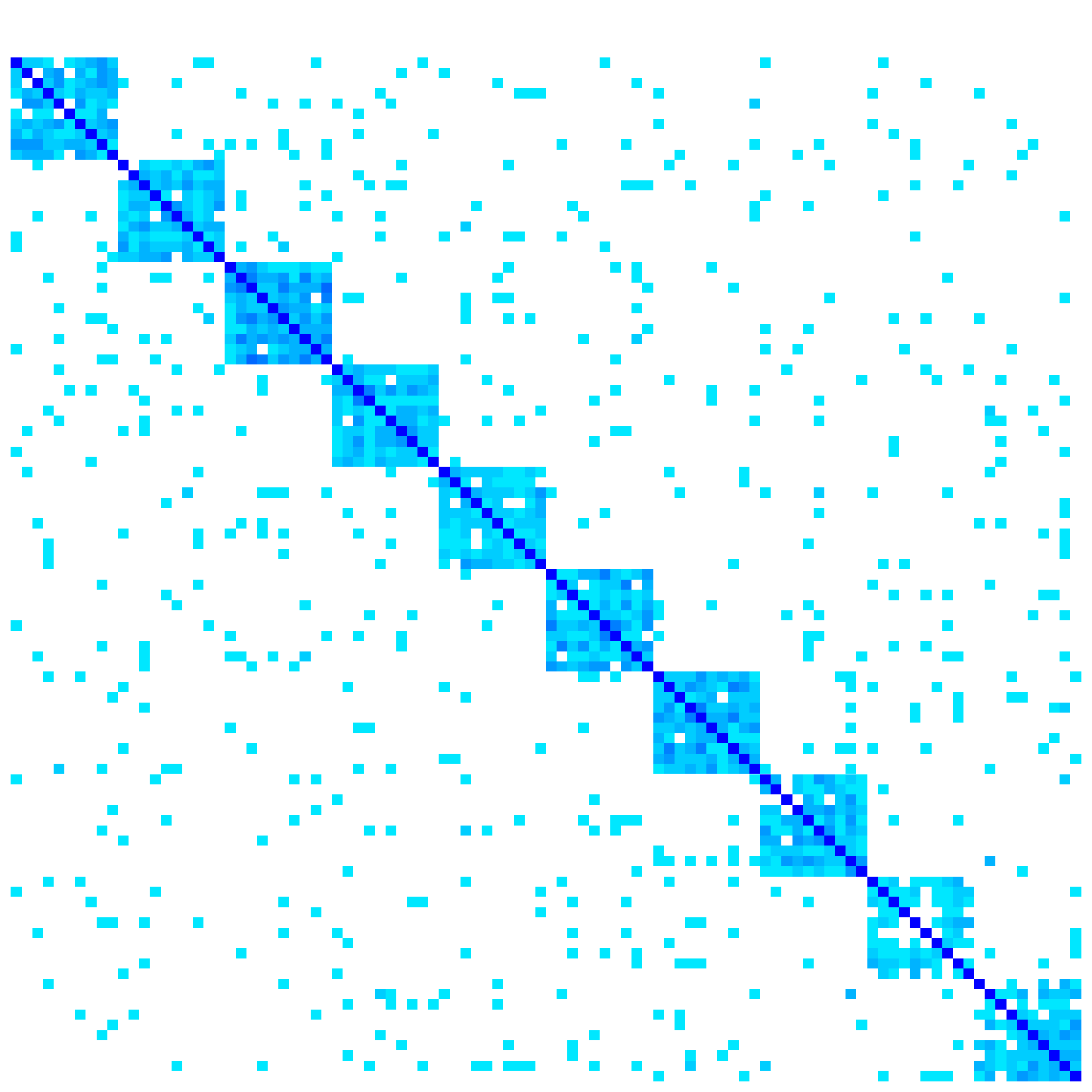}
      \includegraphics[width=0.32\textwidth]{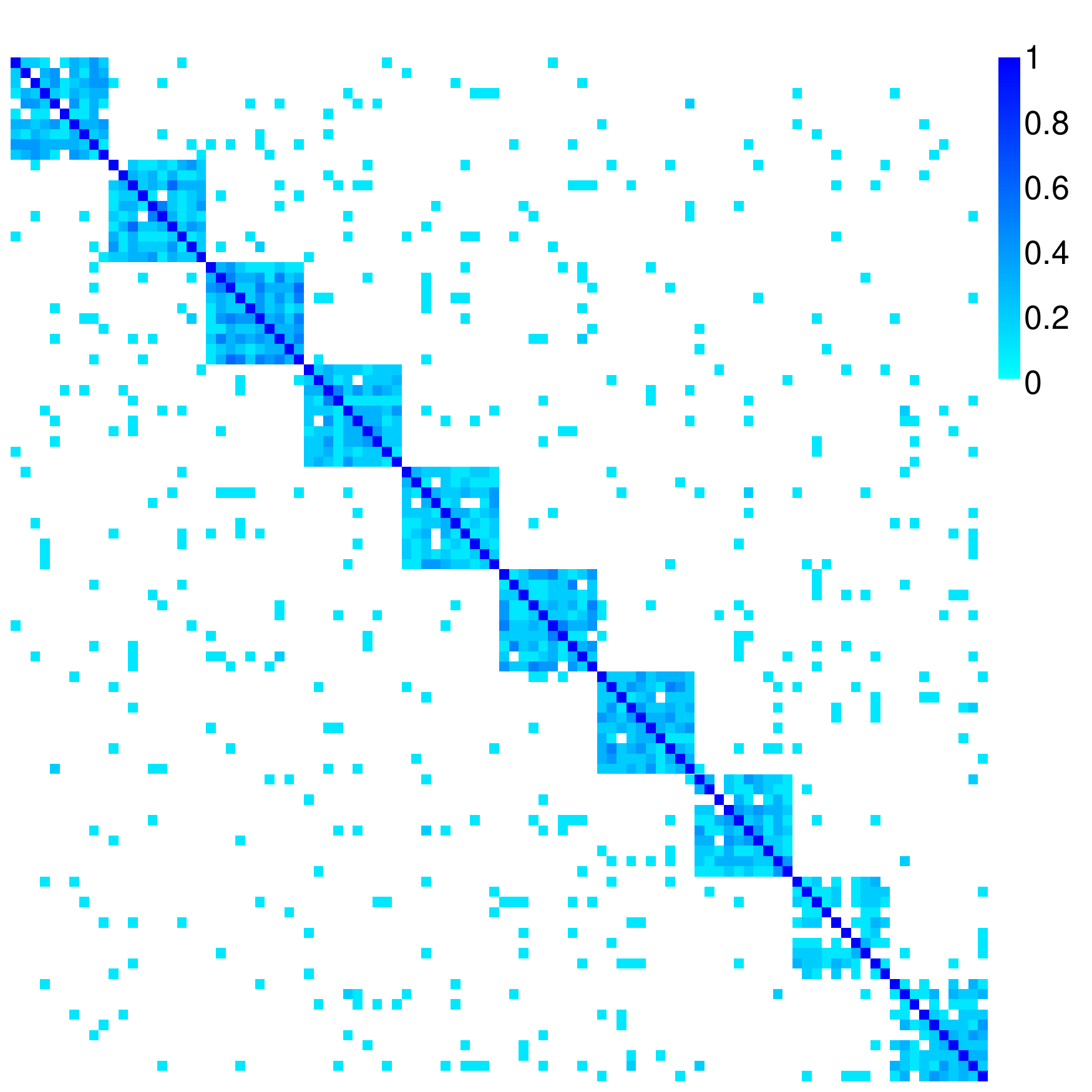}
        (d) \hspace{44mm} (e) \hspace{44mm} (f)\\
  \vspace{2mm}
    \includegraphics[width=0.32\textwidth]{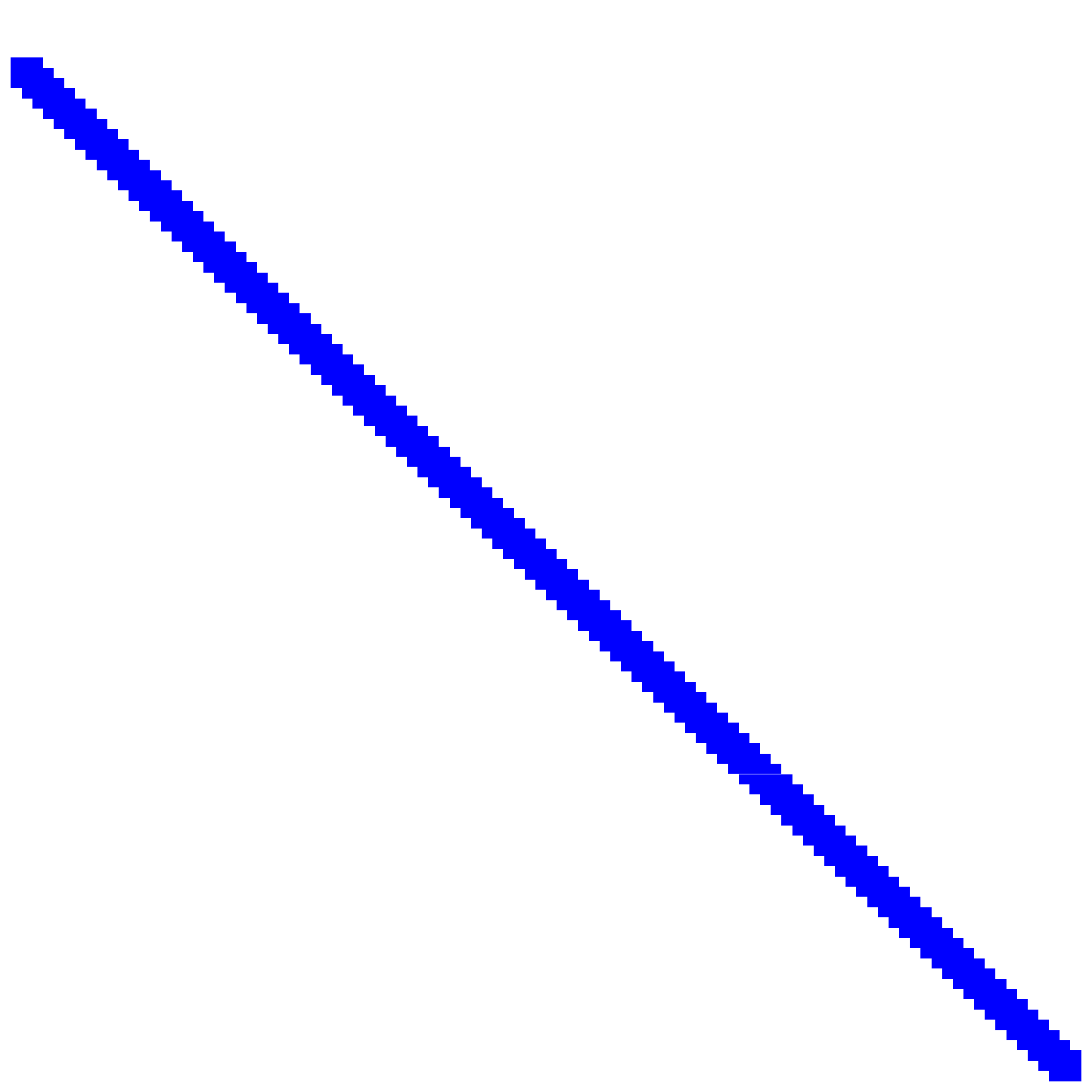}
     \includegraphics[width=0.32\textwidth]{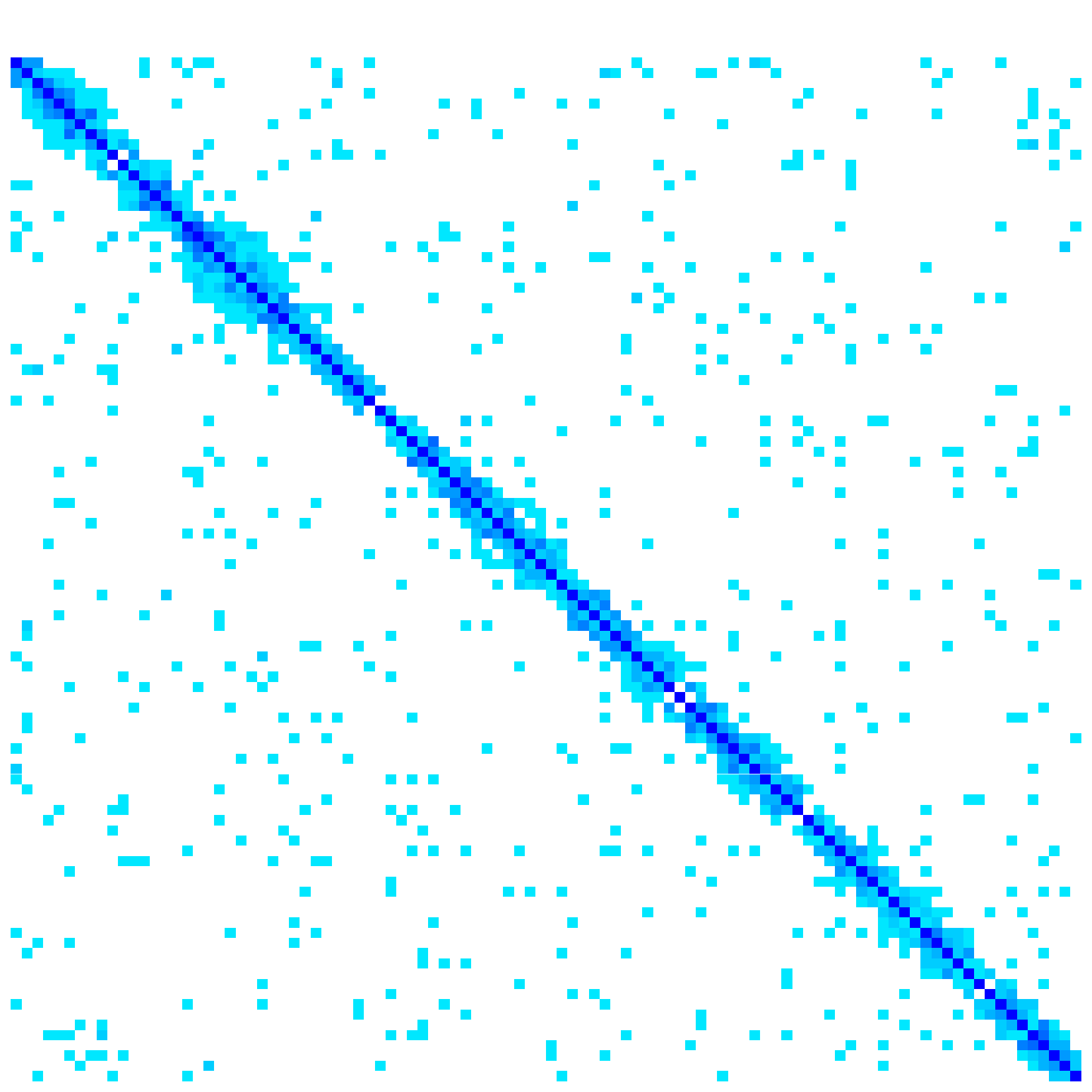}
      \includegraphics[width=0.32\textwidth]{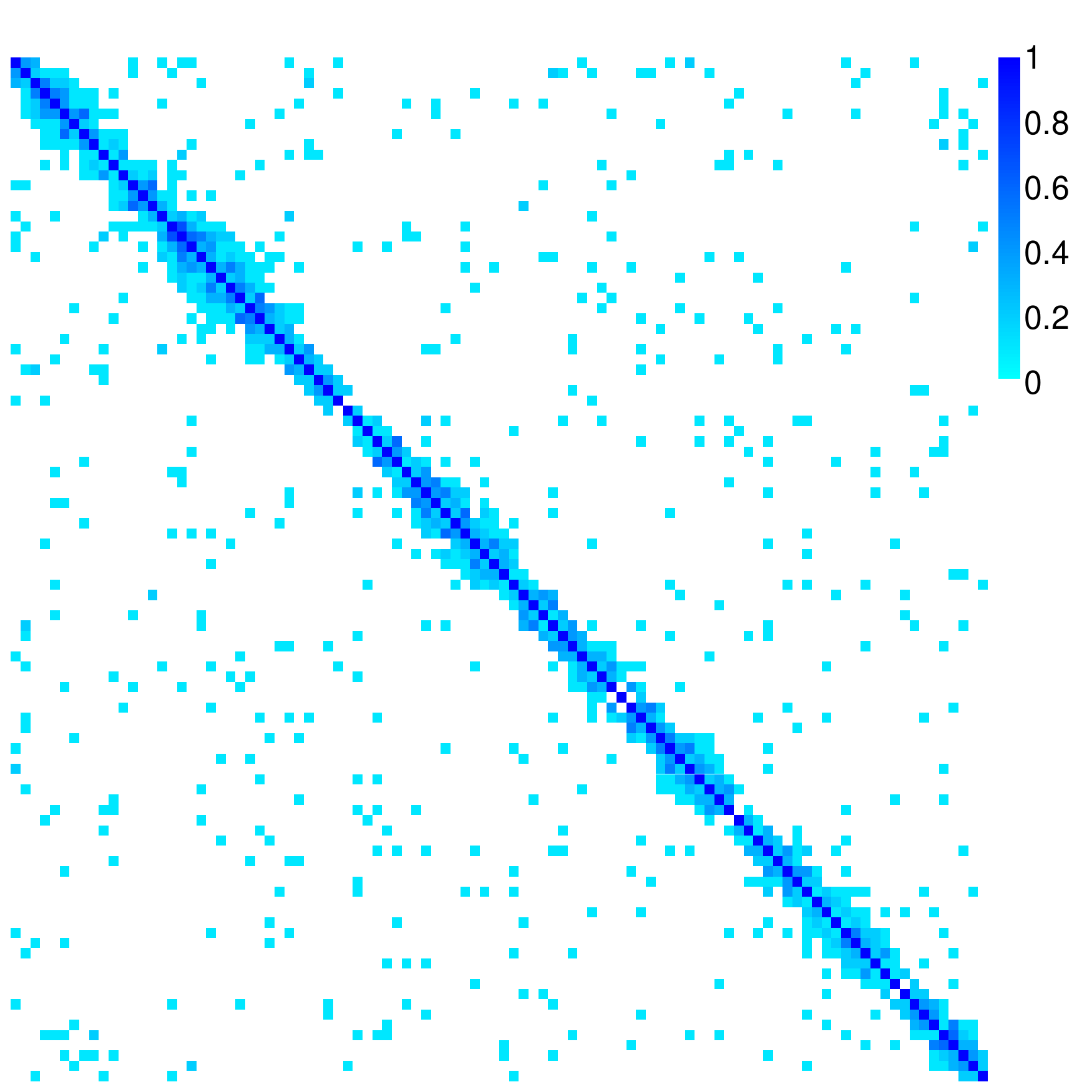}
       \caption{   \label{results:heatmaps} For Simulations B (panels (a)-(c)) and C (panels (d)-(f)) with $p=100$ and $n=50$, the adjacency matrices corresponding to  the true edge set (panels (a) and (d)), the graphical lasso  estimate (panels (b) and (e)), and the GRASS  estimate (panels (c) and (f)) are shown. The adjacency matrices for graphical lasso and GRASS  are averaged over  10 simulated data sets;  the color of a particular cell in the heatmap corresponds to the fraction of these 10 data sets for which the corresponding edge is estimated to be present.  Results for Simulation A are not shown, since in that setting the true edge set is not fixed across the simulated data sets. }
 \end{figure}

\section{Analysis of Gene Expression Data}\label{sec:real}

We examined a gene expression data set from \cite{spira2007airway}, previously studied in \cite{danaher2013joint}, and publicly available from the Gene Expression Omnibus \citep{barrett2007ncbi}  at accession number GDS2771. The data consist of 22,283 microarray-derived gene expression measurements from large airway epithelial cells sampled from 97 patients with lung cancer and 90 controls. We limited our analysis to the 1,778 genes with the highest marginal variance. 
 Each feature was standardized within each class, in order to have mean zero and standard deviation one among the cancer samples, and among the controls. 

Our goal is to compare the performances of GRASS and graphical lasso in terms of edge set recovery. Unfortunately, as is typically the case for a high-dimensional biological data set, the true underlying conditional 
dependence relationships in this data are unknown. In other words, no \emph{gold standard} is available.

Given the absence of a gold standard, we  took the following approach. We split the control samples into two equally-sized sets, Set 1 and Set 2. We applied both graphical lasso and GRASS to each set. 
We refer to the resulting estimated edge sets as $\eonegl$, $\eonesis$, $\etwogl$, and $\etwosis$; tuning parameters  were chosen such that  $|\eonegl|=|\etwogl|=|\eonesis|=|\etwosis|$.
In order to quantify the accuracy of the edges estimated on Set 2 by GRASS and the graphical lasso, we first treated the edges estimated by the graphical lasso on Set 1 as the gold standard, and  then we treated the edges estimated by GRASS on Set 1 as the gold standard. In greater detail, we calculated:
\begin{enumerate}
\item \emph{Accuracy of GRASS on Set 2 when graphical lasso on Set 1 is treated as the gold standard.} This is calculated as
$$ \left| \eonegl \cap \etwosis \cap \left(\etwogl\right)^c \right|.$$
\item \emph{Accuracy of graphical lasso on Set 2 when graphical lasso on Set 1 is treated as the gold standard.} This is calculated as
 $$ \left| \eonegl \cap \etwogl \cap \left(\etwosis\right)^c \right|.$$
\item \emph{Accuracy of GRASS on Set 2 when GRASS  on Set 1 is treated as the gold standard.} This is calculated as
$$ \left| \eonesis \cap \etwosis \cap \left(\etwogl\right)^c \right|.$$
\item \emph{Accuracy of graphical lasso on Set 2 when GRASS on Set 1 is treated as the gold standard.} This is calculated as
 $$ \left| \eonesis \cap \etwogl \cap \left(\etwosis\right)^c \right|.$$
\end{enumerate}
 Note that in calculating these accuracies, we only considered feature pairs $(i,j)$ for which GRASS and graphical lasso disagree over whether an edge is present in Set 2 (since feature pairs for which graphical lasso and GRASS agree on Set 2 are uninformative for our purposes).

The results, averaged over 20 splits of the data into Set 1 and Set 2, are summarized in Table~\ref{tab:realdata}. They indicate that regardless of whether graphical lasso or GRASS  on Set 1 is treated as the gold standard, the results obtained by GRASS on Set 2 have better agreement with the gold standard than do the results obtained by the graphical lasso on Set 2. In other words,  independent data provides greater evidence for edges estimated by GRASS  than for edges estimated by the graphical lasso, regardless of whether the independent data is evaluated using GRASS or  the graphical lasso.

\begin{table}[h!]
\small
  \begin{center}
    \begin{tabular}{c|cc|cc}
    \hline
   & \multicolumn{2}{c|}{GRASS as Gold Standard} & \multicolumn{2}{c}{Graphical Lasso as Gold Standard} \\
    \cline{2-5}
 $|\widehat{\mathcal{E}}|$ &GRASS Accuracy	 &GL Accuracy &   GRASS Accuracy	 & GL Accuracy \\
\hline
 47371.8 (2387.7) &      3368.2 (284.6) &        1663 (185.4) &       2248.7 (85.9) & 1824 (64.5) \\
40781.3 (2346.6) &      2307.4 (226.9) &        1222.8 (154.8) &             1562.5 (68) &   1207 (45.7) \\
33555.9 (2229.2) &      1433.4 (165.4) &        808.3 (119) &         968.5 (48.5) &  714.5 (31.4) \\
25942.8 (2012.8) &      783.4 (108.6) & 472.8 (82.5) &        523.4 (27.3) &  373.1 (17.4) \\
18540.5 (1688.2) &      359.4 (60) &    227 (47.5) &        232.6 (14.9) &  155.4 (8.7) \\
11903.3 (1276.8) &      133.4 (24.6) &  94 (22.1) &          81.6 (5.5) &    58.2 (3.7) \\
6647.9 (834.5) &        41.2 (9.7) &    26.4 (9.1) &        23.6 (2) &      12.7 (1.3) \\
3095.2 (447.3) &        6.2 (2) &       5.5 (1.8) &           2.9 (0.4) &     2.1 (0.3) \\
1142 (184.7) &  1.1 (0.3) &     0.7 (0.5) &      0.5 (0.2) &     0.1 (0.1) \\
312.9 (52.1) &  0.2 (0.1) &     0 (0) &   0 (0) & 0 (0) \\
    \hline
    \end{tabular}
  \end{center}
 \caption{Mean (and standard error) of accuracy  of graphical lasso (GL) and GRASS on gene expression data, over 20 splits of the observations into Set 1 and Set 2. $|\widehat{\mathcal{E}}|$, the size of the estimated edge set, is also reported. Regardless of whether GRASS or graphical lasso is treated as the gold standard, GRASS yields more accurate edge set recovery than does the graphical lasso. These results are based on an analysis of the control observations. Similar results are obtained from the cases (results not shown).  \label{tab:realdata} }
\end{table}

\section{Discussion}\label{sec:disc}
In this paper, we have proposed graphical sure screening (GRASS), a simple and efficient procedure for recovering the   structure of a high-dimensional Gaussian graphical model. GRASS is a natural extension of sure screening approaches from the regression and classification frameworks into the setting of Gaussian graphical modeling.

The theoretical results presented in Section~\ref{sec:theory} for GRASS require a very simple set of assumptions. 
In particular, Assumption~\ref{a1}, which guarantees that the covariance corresponding to an edge in the graph is not too small,  suffices to ensure that GRASS has the sure screening property: that is,  GRASS can dramatically reduce the size of the potential edge set  while still containing the true edge set with probability tending to one. 
Unlike  \citet{ravikumar2011high} and \citet{meinshausen2006high}, no  irrepresentablility conditions  are required. And unlike  \citet{zhou2009adaptive}, no restricted eigenvalue condition is required on the precision matrix. 

The computational advantages of the GRASS framework over existing approaches to estimate a sparse precision matrix are dramatic: while approaches such as the graphical lasso typically require $\mathcal{O}(p^3)$ operations, GRASS requires only $\mathcal{O}(p^2)$ operations.

 In practice, in  order for GRASS to perform well, the non-zero elements of the precision matrix must tend to  be non-zero in the covariance matrix.  We have shown that this assumption typically holds in a range of simulation settings. Furthermore, GRASS outperforms the graphical lasso on a gene expression data set.

\section*{Acknowledgments}
This work was supported in part by NSF Grants DMS 1007698 (R.S.) and DMS CAREER 1252624 (D.W.), NIH Grant 5DP5OD009145 (D.W.), and a Sloan Research Fellowship (D.W.).

\newpage

\appendix
\section{Appendix}

First, we reproduce a result from  \citet{vw96} for the sake of readability.

\begin{lemma}[Bernstein's inequality, Lemma 2.2.11, \citet{vw96}]\label{lem-vw2}
Let $Y_1,\cdots,Y_n$ be independent random variables with zero mean such that $E|Y_i|^m\leq
m!M^{m-2}v_i/2$, for every $m\geq 2$ (and all $i$) and some constants $M$ and $v_i$. Then
$$P\left(|Y_1+\cdots+Y_n|>x\right)\leq 2\exp\{-x^2/(2(v+Mx))\}$$
for $v\geq v_1+\cdots+ v_n$.
\end{lemma}
\bigskip

\bigskip

The following lemma, needed for the proof of Theorem~\ref{thm1}, shows that a $\chi_1^2$-distributed random variable satisfies the moment condition in Lemma \ref{lem-vw2}.
\begin{lemma}\label{lemma1}
   Suppose $Y\sim\chi_1^2$. Then for some constant $C$, we have for all $m\in \mathbb{N}$ that
   $$E|Y-E(Y)|^m \leq Cm!2^m.$$
\end{lemma}

\begin{proof}
  For $Y\sim \chi^2_1$, we have that
\begin{eqnarray*}
  E|Y-E(Y)|^m&=&\frac{1}{\sqrt{2\pi}}\int_{0}^{\infty}|y-1|^my^{-1/2}\exp(-y/2)\mathrm{d}y\\
  &=&\frac{1}{\sqrt{2\pi}}\int_{1}^{\infty}(y-1)^my^{-1/2}\exp(-y/2)\mathrm{d}y\\
    &&+\frac{1}{\sqrt{2\pi}}\int_{0}^{1}(1-y)^my^{-1/2}\exp(-y/2)\mathrm{d}y\\
    &\leq& \frac{1}{\sqrt{2\pi}}\int_{1}^{\infty}(y-1)^m\exp(-y/2)\mathrm{d}y
+\frac{1}{\sqrt{2\pi}}\int_{0}^{1}y^{-1/2}\mathrm{d}y,\\
  &=& \frac{m!2^{m+1}}{ {\exp(1/2)} \sqrt{2\pi} } + \frac{2}{\sqrt{2\pi}}.
 \end{eqnarray*}
 This is not greater than $Cm!2^m$ for some constant $C$.
\end{proof}

\noindent {\bf Proof of Theorem 1}
\begin{proof}

To begin, we will show that there exist constants $c_1$ and $c_2$ such that for any $a \neq b$,  
\begin{equation}
P(|\bm{X}_a^T\bm{X}_b/n-\sigma_{ab}| \geq 1/3C_1n^{-\kappa}) \leq c_1\mathrm{exp}(-c_2n^{1-2\kappa}).
\label{star}
\end{equation}
By definition,
  \begin{eqnarray*}
  &&P\left(|\bm{X}_a^T\bm{X}_b/n-\sigma_{ab}| \geq 1/3C_1n^{-\kappa}\right)\\
  &=& P\left(|\sum_{i=1}^n (X_{ia}X_{ib}-\sigma_{ab})| \geq 1/3C_1n^{1-\kappa}\right)\\
  &=&P\left(|\sum_{i=1}^n[(X_{ia}+X_{ib})^2-2(1+\sigma_{ab})]-\sum_{i=1}^n[(X_{ia}-X_{ib})^2-2(1-\sigma_{ab})]| \geq 4/3C_1n^{1-\kappa}\right). 
  \end{eqnarray*}
  This can be bounded above by 
\begin{eqnarray*}
  && P\left(|\sum_{i=1}^n[(X_{ia}+X_{ib})^2-2(1+\sigma_{ab})]| \geq 2/3C_1n^{1-\kappa}\right)+\\
  &&P\left(|\sum_{i=1}^n[(X_{ia}-X_{ib})^2-2(1-\sigma_{ab})]| \geq 2/3C_1n^{1-\kappa}\right)\\
  &=& P\left(|\sum_{i=1}^n(V_i^2-1)| \geq \frac{1/3C_1n^{1-\kappa}}{(1+\sigma_{ab})}\right)+P\left(|\sum_{i=1}^n(W_i^2-1)| \geq \frac{1/3C_1n^{1-\kappa}}{(1-\sigma_{ab})}\right),
  \end{eqnarray*}
  where $V_1,\ldots, V_n, W_1, \ldots, W_n$ are  independent  standard normal random variables. Hence $V_i^2$ and $W_i^2$ follow a  $\chi^2_1$ distribution. Together, Lemmas \ref{lem-vw2} and \ref{lemma1} guarantee that there exists positive constants $c_1$ and $c_2$ such that
  $$P\left(|\bm{X}_a^T\bm{X}_b/n-\sigma_{ab}| \ge 1/3C_1n^{-\kappa}\right) \leq c_1\mathrm{exp}(-c_2n^{1-2\kappa}).$$

Next, we notice that  
\begin{eqnarray*}
P(\mathcal{E}\not\subseteq \widehat{\mathcal{E}}_{\gamma_n}) &=& P(\bigcup_{(a,b) \in \mathcal{E}} \left\{ |\bm{X}_a^T \bm{X}_b/n| < 2/3 C_1 n^{-\kappa} \right\}) \\
& \leq & \sum_{(a,b) \in \mathcal{E}}  P( |\bm{X}_a^T \bm{X}_b/n| < 2/3 C_1 n^{-\kappa}).
\end{eqnarray*}
Now we note that $|\mathcal{E}| < p^2$, and that for $(a,b) \in \mathcal{E}$, Assumption~\ref{a1} implies that 
$$P( |\bm{X}_a^T \bm{X}_b/n| < 2/3 C_1 n^{-\kappa}) \leq P( |\bm{X}_a^T \bm{X}_b/n - \sigma_{ab} | \ge  1/3 C_1 n^{-\kappa}).$$
Hence, 
$$P(\mathcal{E}\not\subseteq \widehat{\mathcal{E}}_{\gamma_n}) \leq p^2 c_1\mathrm{exp}(-c_2n^{1-2\kappa}).$$
This implies the first part of Theorem~\ref{thm1}. The second part of Theorem~\ref{thm1} can be established in a similar way, and hence we omit the details.
\end{proof}

\bigskip
The following lemma will be used in the proof of Theorem~\ref{thm2}. 
\begin{lemma}\label{lemma3}
Let  $X = (X_1,\dots,X_p)^T$ be a $p$-dimensional random vector with mean zero and variance $\Sigma$, and $Y$ be  a random variable with $E(Y)=0$ and $E(Y^2)=1$. Assume that $X$ and $Y$ satisfy the linear model $Y=X^T\beta+\epsilon$, where  $\epsilon$ is uncorrelated with $X$.  
Define
  $$\mathcal{S} = \{j: |E(YX_j)| {>} Cn^{-\kappa}\}$$
  for some  constant $C$. Then $|\mathcal{S}|$, the cardinality of $\mathcal{S}$,  satisfies 
  $$ |\mathcal{S}| \leq C^{-2}n^{2\kappa}\lambda_{\mathrm{max}}(\Sigma).$$
\end{lemma}

\begin{proof}
{Left-multiplying both sides by $X$ and taking the expectation, } 
$\beta = \Sigma^{-1}E(XY)$. Therefore $E(YX_j) =  (\Sigma \beta)_j$, the $j$th element of the vector $\Sigma \beta$. 
Consequently, 
\begin{equation}
\mathcal{S}=\{j: |(\Sigma \beta)_j |>Cn^{-\kappa}\} = \{j: (\Sigma \beta)_j^2 >C^2n^{-2\kappa}\} . \label{newstar}
\end{equation} 
Furthermore, 
$$||\Sigma\beta||_2^2 = (\Sigma^{1/2}\beta)^T\Sigma(\Sigma^{1/2}\beta) \leq \lambda_{\mathrm{max}}(\Sigma)||\Sigma^{1/2}\beta||_2^2 = \lambda_{\mathrm{max}}(\Sigma)\beta^T\Sigma\beta.$$ Moreover, recalling that $X$ and $\epsilon$ are uncorrelated, we have that $$\beta^T\Sigma\beta = \mbox{Var}(X^T\beta) = \mbox{Var}(Y)-\mbox{Var}(\epsilon) \leq 1.$$ Thus,  we conclude that $||\Sigma\beta||_2^2 \leq \lambda_{\mathrm{max}}(\Sigma)$. By (\ref{newstar}), this implies that $|\mathcal{S}| \leq \lambda_{\max}(\Sigma) / (C^2n^{-2\kappa}) = C^{-2}n^{2\kappa}\lambda_{\mathrm{max}}(\Sigma)$.

\end{proof}

\bigskip

\noindent {\bf Proof of Theorem 2}

\begin{proof}
Let $$\mathcal{S}_a = \{b: b\not=a,  |\sigma_{ab}|\ge 1/3C_1n^{-\kappa} \}$$ and 
$$\mathcal{T}_{a,\gamma_n} = \bigcap_{b: b\not=a} \{ |{\bf X}_a^T{\bf X}_b/n-\sigma_{ab}|\leq 1/3C_1n^{-\kappa} \}.$$

 By definition,  $\widehat{\mathcal{E}}_{a,\gamma_n} = \{b: b\neq a, |\bm{X}_a^T \bm{X}_b/n| > 2/3C_1n^{-\kappa}\}$. Then on the set $\mathcal{T}_{a,\gamma_n}$, if $b$ belongs to $\widehat{\mathcal{E}}_{a,\gamma_n}$, it has to belong to $\mathcal{S}_a$. Thus, we conclude that $P(\widehat{\mathcal{E}}_{a,\gamma_n} \subseteq \mathcal{S}_a) \geq P(\mathcal{T}_{a,\gamma_n})$. Moreover, an argument similar to that in the proof of Theorem \ref{thm1} can be used to show that $$P(\mathcal{T}_{a,\gamma_n}) \geq 1 - C_4\mathrm{exp}(-C_5n^{1-2\kappa}).$$ This implies that $$P(\widehat{\mathcal{E}}_{a,\gamma_n} \subseteq \mathcal{S}_a) \geq 1 - C_4\mathrm{exp}(-C_5n^{1-2\kappa}).$$
  Finally, applying 
   {Lemma \ref{lemma3} in conjunction with Assumption~\ref{a2}} yields the desired result.
\end{proof}

\noindent {\bf Proof of Theorem 3}
\begin{proof}

First, we will show that the assumptions of Theorem~\ref{thm1} are satisfied, so that the sure screening property applies. We must simply show that the new threshold, 
$\gamma_n = \Phi^{-1} (1- \frac{f}{p(p-1)}) / \sqrt{n}$,
is no greater than $2/3C_1n^{-\kappa}$, the threshold used in Theorem~\ref{thm1}. In other words, we must show that 
\begin{equation}
\frac{f}{p(p-1)} \geq 1-\Phi(2/3C_1n^{1/2-\kappa}). \label{wts}
\end{equation}
From the fact that 
$$1-\Phi(x) \leq \frac{1}{\sqrt{2\pi}}x^{-1}\exp(-x^2/2),$$ we have that $1-\Phi(2/3C_1n^{1/2-\kappa}) {\leq} C_7n^{-1/2+\kappa}\exp(-C_8n^{1-2\kappa})$.
Furthermore, {since $\log(p)=C_3 n^\xi$}, we have that  $\frac{f}{p(p-1)} {\geq} C_9\exp(-C_{10}n^{\xi})$.   Using the fact that $\xi < 1-2\kappa$, (\ref{wts}) follows directly.

Next, we show that using a threshold value of $\gamma_n = \Phi^{-1} (1- \frac{f}{p(p-1)}) / \sqrt{n}$ leads to control of the asymptotic {expected} false positive rate at $f/[p(p-1)]$. 
Recall that the false positive rate is defined as
\begin{eqnarray*}
  \mbox{fpr}_n = \frac{1}{|\mathcal{E}^c|} \sum_{(a,b)\not\in \mathcal{E}}{\bf 1}\left(|\frac{\bm{X}_a^T\bm{X}_b}{n}| > \gamma_n\right).
\end{eqnarray*}
Because $\mbox{E}(\bm{X}_a^T\bm{X}_b/n) = \sigma_{ab}$ and $\mbox{Var}(\bm{X}_a^T\bm{X}_b/n) = \frac{1 + \sigma_{ab}^2}{n}$, it follows that 
$$
\sqrt{n}\left(\frac{\bm{X}_a^T\bm{X}_b}{n}-\sigma_{ab}\right)/\sqrt{1+\sigma_{ab}^2} {\rightarrow_d} N(0,1).
$$
Furthermore, for any $(a,b)\not\in \mathcal{E}$, we have
\begin{eqnarray*}
&&P\left(|\frac{\bm{X}_a^T\bm{X}_b}{n}| > \gamma_n\right) = P\left(\frac{\sqrt{n}(\frac{\bm{X}_a^T\bm{X}_b}{n}-\sigma_{ab})}{\sqrt{1+\sigma_{ab}^2}} > \frac{\sqrt{n}\gamma_n-\sqrt{n}\sigma_{ab}}{\sqrt{1+\sigma_{ab}^2}}\right) \\
&& + P\left(\frac{\sqrt{n}(\frac{\bm{X}_a^T\bm{X}_b}{n}-\sigma_{ab})}{\sqrt{1+\sigma_{ab}^2}} < -\frac{\sqrt{n}\gamma_n+\sqrt{n}\sigma_{ab}}{\sqrt{1+\sigma_{ab}^2}}\right)\\
&& = 1-\Phi\left(\frac{\sqrt{n}\gamma_n-\sqrt{n}\sigma_{ab}}{\sqrt{1+\sigma_{ab}^2}}\right) + 1- \Phi\left(\frac{\sqrt{n}\gamma_n+\sqrt{n}\sigma_{ab}}{\sqrt{1+\sigma_{ab}^2}}\right) \\
&& \asymp 2-2\Phi(\sqrt{n}\gamma_n) = 2f/[p(p-1)]. 
\end{eqnarray*}
where the  asymptotic equivalence in the previous line follows from the fact that $\sqrt{n} \gamma_n = \Phi^{-1}(1-f/[p(p-1)]) $ is of the same order as $n^{\frac{\xi}{2}}$, 
combined with Assumption~\ref{a3}.

Consequently, the expectation of $\mbox{fpr}_n$ is controlled as desired,
\begin{eqnarray*}
  \mbox{E}(\mbox{fpr}_n) &=&  \frac{1}{|\mathcal{E}^c|} \sum_{(a,b)\not\in \mathcal{E}} P\left(|\frac{\bm{X}_a^T\bm{X}_b}{n}| > \gamma_n\right) \\
   &\asymp& \frac{\sum_{(a,b)\not\in \mathcal{E}} 2f / [p(p-1)] }{|\mathcal{E}^c|} =  2f/[p(p-1)] \leq f/|\mathcal{E}^c|,
\end{eqnarray*}
where the last inequality is due to the fact that $|\mathcal{E}^c| \leq p(p-1)/2$.

\end{proof}

\newpage

\bibliographystyle{apalike}
\bibliography{prop-ref}

\end{document}